\documentclass[twoside]{article}

\usepackage[accepted]{aistats2019}
\usepackage[ruled]{algorithm}
\usepackage{algorithmicx}    
\usepackage{algpseudocode}
\usepackage{graphicx}
\usepackage{epstopdf}
\usepackage{subcaption}
\usepackage{amsthm,amsfonts}
\usepackage{bbm, xcolor}
\usepackage{enumitem}

\usepackage{bm}
\usepackage{comment}
\newcommand{\R}{\mathcal{R}}
\newcommand{\ie}{\textit{i.e.}}
\newcommand{\eg}{\textit{e.g.}}
\newcommand{\Ex}{\mathbf{E}}

\newcommand{\argmax}{\text{argmax}}

\newcommand{\swofu}{\texttt{SW-UCB}\text{ algorithm}}

\newcommand{\bob}{\texttt{BOB}\text{ algorithm}}
\newtheorem{theorem}{Theorem}
\newtheorem{lemma}{Lemma}
\newtheorem{assumption}{Assumption}

\allowdisplaybreaks
\begin{document}

%

%

\twocolumn[

\aistatstitle{Learning to Optimize under Non-Stationarity}

\aistatsauthor{Wang Chi Cheung \And  David Simchi-Levi \And Ruihao Zhu}

\aistatsaddress{ISEM, NUS \And IDSS, MIT\And SDSC, MIT} 
]

\begin{abstract}
	We introduce algorithms that achieve state-of-the-art \emph{dynamic regret} bounds for non-stationary linear stochastic bandit setting. It captures natural applications such as dynamic pricing and ads allocation in a changing environment. We show how the difficulty posed by the non-stationarity can be overcome by a novel marriage between stochastic and adversarial bandits learning algorithms. Defining $d,B_T,$ and $T$ as the problem dimension, the \emph{variation budget}, and the total time horizon, respectively, our main contributions are the tuned Sliding Window UCB (\texttt{SW-UCB}) algorithm with optimal $\widetilde{O}(d^{2/3}(B_T+1)^{1/3}T^{2/3})$ dynamic regret, and the tuning free bandit-over-bandit (\texttt{BOB}) framework built on top of the \texttt{SW-UCB} algorithm with best $\widetilde{O}(d^{2/3}(B_T+1)^{1/4}T^{3/4})$ dynamic regret.
\end{abstract}

\section{Introduction}

Multi-armed bandit (MAB) problems are online problems with partial feedback, when the learner is subject to uncertainty in his/her learning environment. Traditionally, most MAB problems are studied in the stochastic \cite{ABF02} and adversarial \cite{ABFS02} environments. In the former, the model uncertainty is static and the partial feedback is corrupted by a mean zero random noise. The learner aims at estimating the latent static environment and converging to a static optimal decision. In the latter, the model is dynamically changed by an adversary. The learner strives to hedge against the changes, and compete favorably in comparison to certain benchmark policies.

While assuming a stochastic environment could be too simplistic in a changing world, sometimes the assumption of an adversarial environment could be too pessimistic. Recently, a stream of research works (see Related Works) focuses on MAB problems in a \emph{drifting} environment, which is a hybrid of a stochastic and an adversarial environment. Although the environment can be dynamically and adversarially changed, the total change (quantified by a suitable metric) in a $T$ step problem is upper bounded by $B_T~(= \Theta(T^{\rho})\text{ for some }\rho\in (0, 1))$, the \emph{variation budget}. The feedback is corrupted by a mean zero random noise. The aim is to minimize the \emph{dynamic regret}, which is the optimality gap compared to the sequence of (possibly dynamically changing) optimal decisions, by simultaneously estimating the current environment and hedging against future changes every time step. Most of the existing works for non-stationary bandits have focused on the the somewhat ideal case in which $B_T$ is known. In practice, however, $B_T$ is often not available ahead. Though some efforts have been made towards this direction \cite{KA16,LWAL18}, how to design algorithms with low dynamic regret when $B_T$ is unknown remains largely as a challenging problem.

In this paper, we design and analyze novel algorithms for the linear bandit problem in a drifting environment. Our main contributions are listed as follows.
\begin{itemize}[leftmargin=*,itemsep=-1.5mm,topsep=-1.5mm]
	\item When the variation budget $B_T$ is known, we characterize the lower bound of dynamic regret, and develop a tuned Sliding Window UCB (\texttt{SW-UCB}) algorithm with matched dynamic regret upper bound up to  logarithmic factors.
	\item When $B_T$ is unknown, we propose a novel Bandit-over-Bandit (\texttt{BOB}) framework that tunes \texttt{SW-UCB} adaptively. The application of \texttt{BOB} on \swofu~ achieves the best dependence on $T$ compared to existing literature.
\end{itemize} 
\vspace{-3mm}
\paragraph{Related Works. } MAB problems with stochastic and adversarial environments are extensively studied, as surveyed in \cite{BC12, LS18}. 
To model inter-dependence relationships among different arms, models for linear bandits in stochastic environments have been studied. In \cite{A02,DHK08,RT10,CLRS11,AYPS11}, UCB type algorithms for stochastic linear bandits were studied, and Abbasi-Yadkori et al. \cite{AYPS11} possessed the state-of-art algorithm for the problem. Thompson Sampling algorithms proposed in \cite{RVR14,AG13,AL17} are able to bypass the high computational complexities provided that one can efficiently sample from the posterior on the parameters and optimize the reward function accordingly. Unfortunately, achieving optimal regret bound via TS algorithms is possible only if the true prior over the reward vector is known. 

Authors of \cite{BGZ14,BGZ18} considered the $K$-armed bandits in a drifting environment. They achieved the tight dynamic regret bound $\tilde{O}((K B_T)^{1/3} T^{2/3})$ when $B_T$ is known. Wei et al. \cite{WHL16} provided refined regret bounds based on empirical variance estimation, assuming the knowledge of $B_T$. Subsequently, Karnin et al. \cite{KA16} considered the setting without knowing $B_T$ and $K=2$, and achieved a dynamic regret bound of $\tilde{O}(B_T^{0.18} T^{0.82} + T^{0.77})$. In a recent work, \cite{LWAL18} considered $K$-armed contextual bandits in drifting environments, and in particular demonstrated an improved bound $\tilde{O}(K B_T^{1/5}T^{4/5})$ for the $K$-armed bandit problem in drifting environments when $B_T$ is not known, among other results. \cite{KZ16} considered a dynamic pricing problem in a drifting environment with linear demands. Assuming a known variation budget $B_T,$ they proved an $\Omega (B_T^{1/3}T^{2/3} )$ dynamic regret lower bound and proposed a matching algorithm. When $B_T$ is not known, they designed an algorithm with $\tilde{O}(B_T T^{2/3})$ dynamic regret. In \cite{BGZ15}, a general problem of stochastic optimization under the known budgeted variation environment was studied. The authors presented various upper and lower bound in the full feedback settings. Finally, various online problems with full information feedback and drifting environments are studied in the literature \cite{CYLMLJZ12, JRSS15}. 

Apart from drifting environment, numerous research works consider the \emph{switching environment}, where the time horizon is partitioned into at most $S$ intervals, and it switches from one stochastic environment to another across different intervals. The partition is not known to the learner. Algorithms are designed for various bandits, assuming a known $S$ \cite{ABFS02, GM11, LWAL18}, or assuming an unknown $S$ \cite{KA16, LWAL18}. Notably, the Sliding Window UCB for the $K$-armed setting is first proposed by Garivier et al. \cite{GM11}, while it is only analyzed under switching environments.

Finally, it is worth pointing out that our Bandits-over-Bandits framework has connections with algorithms for online model selection and bandit corralling, see \eg, \cite{ALNS17} and references therein. This and similar techniques have been investigated under the context of non-stationary bandits in \cite{LWAL18,BGZ18}. Notwithstanding, existing works either have no theoretical guarantee or can only obtain sub-optimal dynamic regret bounds.
	
	\section{Problem Formulation}
	\label{sec:formulation}
	In this section, we introduce the notations to be used throughout the discussions and the model formulation.
	\subsection{Notation}
	Throughout the paper, all vectors are column vectors, unless specified otherwise. We define $[n]$ to be the set $\{1,2,\ldots,n\}$ for any positive integer $n.$ The notation $a:b$ is the abbreviation of consecutive indexes $a,a+1,\ldots, b.$ We use $\|\bm x\|$ to denote the Euclidean norm of a vector $\bm x\in\Re^d.$ 
	For a positive definite matrix $A\in \Re^{d\times d}$, we use $\|\bm x\|_A$ to denote the matrix norm $\sqrt{\bm{x}^{\top}A\bm x}$ of a vector $\bm x\in\Re^d.$ We also denote $x\vee y$ and $x\wedge y$ as the maximum and minimum between $x,y\in\Re,$ respectively. When logarithmic factors are omitted, we use $\widetilde{O}(\cdot)$ to denote function growth.
	\subsection{Learning Model}
In each round $t\in[T]$, a decision set $D_t\subseteq\Re^d$ is presented to the learner, and it has to choose an action $X_t\in D_t.$ Afterwards, the reward 
$Y_t=\langle X_t,\theta_t\rangle+\eta_t$
is revealed. Here, we allow $D_t$ to be chosen by an \emph{oblivious adversary} whose actions are independent of those of the learner, and can be determined before the protocol starts \cite{CBL06}. $\theta_t\in\Re^d$ is an unknown $d$-dimensional vector, and $\eta_t$ is a random noise drawn i.i.d. from an unknown sub-Gaussian distribution with variance proxy $R$. This implies $\Ex\left[\eta_t\right]=0$, and $\forall\lambda\in\Re$ we have
$\Ex\left[\exp\left(\lambda\eta_t\right)\right]\leq\exp\left(\frac{\lambda^2R^2}{2}\right).$
Following the convention of existing bandits literature \cite{AYPS11,AG13}, we assume there exist positive constants $L$ and $S,$ such that $\|X\|\leq L$  and $\|\theta_t\|\leq S$ holds for all $X\in D_t$ and all $t\in[T],$ and the problem instance is normalized so that $\left|\langle X,\theta_t\rangle\right|\leq 1$ for all $X\in D_t$ and $t\in[T].$

Instead of assuming the stochastic environment, where reward function remains stationary across the time horizon, we allow it to change over time. Specifically, we consider the general \emph{drifting environment:} the sum of $\ell_2$ differences of consecutive $\theta_t$'s should be bounded by some variation budget $B_T= \Theta(T^{\rho})\text{ for some }\rho\in (0, 1)$, \ie, \begin{align}\label{eq:variation_budget}\sum_{t=1}^{T-1}\left\|\theta_{t+1}-\theta_t\right\|\leq B_T.\end{align}
We again allow the $\theta_t$'s to be chosen adversarially by an oblivious adversary. We also denote the set of all possible obliviously selected sequences of $\theta_t$'s that satisfies inequality (\ref{eq:variation_budget}) as $\Theta(B_T).$

The learner's goal is to design a policy $\pi$ to maximize the cumulative reward, or equivalently to minimize the worst case cumulative regret against the optimal policy $\pi^*$, that has full knowledge of $\theta_t$'s. Denoting $x_t^*=\argmax_{x\in D_t}\langle x,\theta_t\rangle,$ the dynamic regret of a given policy $\pi$ is defined as
$
	\R_T(\pi)=\sup_{\theta_{1:T}\in\Theta(B_T)}\Ex\left[\sum_{t=1}^T\langle x_t^*-X_t,\theta_t\rangle\right],
$
where the expectation is taken with respect to the (possible) randomness of the policy.
	\section{Lower Bound}
	We first provide a lower bound on the the regret to characterize the best achievable regret.
	\begin{theorem}
		\label{theorem:lower_bound}
		For any $T\geq d,$ the dynamic regret of any policy $\pi$ satisfies $\R_T(\pi)=\Omega\left(d^{\frac{2}{3}}B_T^{\frac{1}{3}}T^{\frac{2}{3}}\right).$
	\end{theorem}
	\begin{proof}[Sketch Proof]
		The construction of the lower bound instance is similar to the approach of \cite{BGZ14}: nature divides the whole time horizon into $\lceil T/H\rceil$ blocks of equal length $H$ rounds (the last block can possibly have less than $H$ rounds). In each block, the nature initiates a new stationary linear bandit instance with parameters from the set $\{\pm\sqrt{d/4H}\}^d.$ Nature also chooses the parameter for a block in a way that depends only on the learner's policy, and the worst case regret is $\Omega(d\sqrt{H}).$ Since there is at least $\lfloor T/H\rfloor$ number of blocks, the total regret is $\Omega(dT/\sqrt{H}).$ By examining the variation budget constraint, we have that the smallest possible $H$ one can take is $\lceil{(dT)^{\frac{2}{3}}B_T^{-\frac{2}{3}}}\rceil.$ The statement then follows. Please refer to Section \ref{sec:theorem:lower_bound} for the complete proof.
	\end{proof}
	\section{Sliding Window Regularized Least Squares Estimator}
	\label{sec:swlse}
	As a preliminary, we introduce the sliding window regularized least squares estimator, which is the key tool in estimating the unknown parameters $\{\theta_t\}^T_{t=1}$. Despite the underlying non-stationarity, we show that the estimation error of this estimator can gracefully adapt to the parameter changes.
	
	Consider a sliding window of length $w,$ and consider the observation history $\{(X_s, Y_s)\}^{t-1}_{s = 1\vee (t-w)}$ during the time window $(1\vee (t-w)) : (t-1)$. The ridge regression problem with regularization parameter $\lambda~(>0)$ is stated below:
	\begin{align}
		\min_{\theta\in\Re^d} \lambda \left\|\theta\right\|^2 + \sum^{t-1}_{s = 1\vee (t-w)} (X_s^\top \theta - Y_s)^2.
	\end{align}
	Denote $\hat{\theta}_t$ as a solution to the regularized ridge regression problem, and define matrix $V_{t-1} := \lambda I + \sum^{t-1}_{s = 1\vee (t-w)} X_s X_s^\top$. The solution $\hat{\theta}_t$ has the following explicit expression:
	\begin{align}
		\label{eq:sw20}
		&\nonumber\hat{\theta}_t = V_{t-1}^{-1}\left( \sum^{t-1}_{s = 1\vee(t-w)}X_s Y_s \right) \\
		=& V_{t-1}^{-1}\left(\sum_{s=1\vee(t-w)}^{t-1}X_sX_s^{\top}\theta_s+\sum_{s=1\vee (t-w)}^{t-1}\eta_sX_s\right).
	\end{align}
	The difference $\hat{\theta}_t-\theta_t=$ has the following expression:
	\begin{align}
		\nonumber&V_{t-1}^{-1}\left(\sum_{s=1\vee(t-w)}^{t-1}X_sX_s^{\top}\theta_s+\sum_{s=1\vee(t-w)}^{t-1}\eta_sX_s\right)-\theta_t\\
		\nonumber=&V_{t-1}^{-1}\sum_{s=1\vee(t-w)}^{t-1}X_sX_s^{\top}\left(\theta_s-\theta_t\right)+V_{t-1}^{-1} \sum_{s=1\vee(t-w)}^{t-1}\eta_sX_s\\
		\label{eq:sw16}&-\lambda\theta_t ,
	\end{align}
	The first term on the right hand side of eq. (\ref{eq:sw16}) is the estimation inaccuracy due to the non-stationarity; while the second term is the estimation error due to random noise. We now upper bound the two terms separately, under the following regularity assumption made in \cite{FauryRAC21} over the decision sets $D_t$'s.
	\begin{assumption}\label{ass:reg}
		There exists an orthonormal basis $\Psi=(\psi_1,\ldots,\psi_d)$ such that for any $t\in[T]$ and any $X\in D_t,$ there exists a number $z\in\mathbb{R}$ and an $i\in[d]$ such that $X=z\cdot \psi_i.$ 
	\end{assumption}
		One can easily verify that this assumption holds in the multi-armed bandits case. Of course, this assumption allows for more general models than the multi-armed bandits setting as it still allows each of the $D_t$'s to have arbitrarily large number of actions. We now upper bound the first term in the $\ell_2$ sense.
	\begin{lemma}
		\label{lemma:sw}
		For any $t\in[T],$ we have
		\begin{align*}
			&\left\|V_{t-1}^{-1}\sum_{s=1\vee(t-w)}^{t-1}X_sX_s^{\top}\left(\theta_s-\theta_t\right)\right\|\\
			\leq&\sum^{t-1}_{s = 1\vee (t-w)}\left\|\theta_s-\theta_{s+1}\right\|.
		\end{align*}
	\end{lemma}
	\begin{proof}[Sketch Proof.] 
		Our analysis relies on bounding the maximum eigenvalue of $V_{t-1}^{-1}\sum_{s=1\vee(t-w)}^{p}X_sX_s^{\top}$ for each $p\in \{1\vee (t-w), \ldots, t-1\}$. Please refer to Section \ref{sec:lemma:sw} of appendix for the complete proof.
	\end{proof}
	Adopting the analysis in \cite{AYPS11}, we upper bound the second term in the matrix norm sense.
	\begin{lemma}[\cite{AYPS11}]
		\label{lemma:sw1}
		For any $t\in[T]$ and any $\delta\in[0,1],$ we have with probability at least $1-\delta,$
		\begin{align*}
			\left\|\sum_{s=1\vee(t-w)}^{t-1}\eta_sX_s-\lambda\theta_t\right\|_{V_{t-1}^{-1}}\leq&R\sqrt{d\ln\left(\frac{1+wL^2/\lambda}{\delta}\right)}\\&+\sqrt{\lambda}S.
		\end{align*}
	\end{lemma}
	From now on, we shall denote
	\begin{equation}\label{eq:sw_beta}
		\beta := R\sqrt{d\ln\left(\frac{1+wL^2/\lambda}{\delta}\right)}+\sqrt{\lambda}S
	\end{equation} 
	for the ease of presentation. With these two lemmas, we have the following deviation inequality type bound for the latent expected reward of any action $x\in D_t$ in any round $t.$ 
	\begin{theorem}
		\label{theorem:sw_deviation}
		For any $t\in[T]$ and any $\delta\in[0,1]$, with probability at least $1-\delta,$ it holds for all $x\in D_t$ that
		\begin{align*}
			&\left|x^\top ( \hat{\theta}_t - \theta_t)\right|\leq L \sum^{t-1}_{s = 1\vee (t-w)}\left\|\theta_s-\theta_{s+1}\right\|+\beta\left\|x\right\|_{V^{-1}_{t-1}}
		\end{align*}
		
	\end{theorem} 
	\begin{proof}[Sketch Proof.]
		The proof is a direct application of Lemmas \ref{lemma:sw} and \ref{lemma:sw1}. Please refer to Section \ref{sec:theorem:sw_deviation} of the appendix for the complete proof.
	\end{proof}
	\section{Sliding Window-Upper Confidence Bound (\texttt{SW-UCB}) Algorithm: A First Order Optimal Strategy}
	\label{sec:swofu}
	In this section, we describe the Sliding Window Upper Confidence Bound (\texttt{SW-UCB}) algorithm. When the variation budget $B_T$ is known, we show that \swofu ~with a tuned window size achieves a dynamic regret bound which is optimal up to a multiplicative logarithmic factor. When the variation budget $B_T$ is unknown, we show that \swofu ~can still be implemented with a suitably chosen window size so that the regret dependency on $T$ is optimal, which still results in first order optimality in this case \cite{KZ16}.
	\subsection{Design Intuition}
	\label{sec:swofu_intuition}
	In the stochastic environment where the linear reward function is stationary, the well known UCB algorithm follows the principle of optimism in face of uncertainty. Under this principle, the learner selects the action that maximizes the UCB, or the value of ``mean plus confidence radius" \cite{ABF02}. We follow the principle by choosing in each round the action $X_t$ with the highest UCB, \ie,
	\begin{align}
		\label{eq:sw_policy}
		\nonumber X_t\nonumber=&\argmax_{x\in D_t}\left\{\langle x,\hat{\theta}_t\rangle\right.\\
		\nonumber&\left.+L \sum^{t-1}_{s = 1\vee (t-w)}\left\|\theta_s-\theta_{s+1}\right\|+\beta\left\|x\right\|_{V^{-1}_{t-1}}\right\}\\
		=&\argmax_{x\in D_t}\left\{\langle x,\hat{\theta}_t\rangle+\beta\left\|x\right\|_{V^{-1}_{t-1}} \right\}.
	\end{align}
	When the number of actions is moderate, the optimization problem (\ref{eq:sw_policy}) can be solved by an enumeration over all $x\in D_t.$ Upon selecting $X_t,$ we have
	\begin{align}
		\label{eq:sw14}
		\nonumber&\langle x^*_t,\hat{\theta}_t\rangle+L \sum^{t-1}_{s = 1\vee (t-w)}\left\|\theta_s-\theta_{s+1}\right\|+\beta\left\|x^*_t\right\|_{V^{-1}_{t-1}}\\
		\leq&\langle X_t,\hat{\theta}_t\rangle+L \sum^{t-1}_{s = 1\vee (t-w)}\left\|\theta_s-\theta_{s+1}\right\|+ \beta\left\|X_t\right\|_{V^{-1}_{t-1}},
	\end{align}
	by virtue of UCB. From Theorem \ref{theorem:sw_deviation}, we further have with probability at least $1-\delta,$
	\begin{align}
		\label{eq:sw18}
		\langle x^*_t,\theta_t - \hat{\theta}_t\rangle \leq L \sum^{t-1}_{s = 1\vee (t-w)}\left\|\theta_s-\theta_{s+1}\right\|+\beta\left\|x^*_t\right\|_{V^{-1}_{t-1}}, 
	\end{align}
	and
	\begin{align}
		\label{eq:sw19}
		\nonumber&\langle X_t,\hat{\theta}_t\rangle+L \sum^{t-1}_{s = 1\vee (t-w)}\left\|\theta_s-\theta_{s+1}\right\|+\beta\left\|X_t\right\|_{V^{-1}_{t-1}}\\
		\leq&\langle X_t,\theta_t\rangle+ 2L \sum^{t-1}_{s = 1\vee (t-w)}\left\|\theta_s-\theta_{s+1}\right\|+ 2\beta\left\|X_t\right\|_{V^{-1}_{t-1}} .
	\end{align}
	Combining inequalities (\ref{eq:sw14}), (\ref{eq:sw18}), and (\ref{eq:sw19}), we establish the following high probability upper bound for the expected per round regret, \ie, with probability $1-\delta,$
	\begin{equation}\label{eq:sw-ucb}
		\langle x^*_t - X_t, \theta_t\rangle \leq 2 L \sum^{t-1}_{s = 1\vee (t-w)}\left\|\theta_s-\theta_{s+1}\right\|+2\beta\left\|X_t\right\|_{V^{-1}_{t-1}} .
	\end{equation}
	The regret upper bound of the \swofu~(to be formalized in Theorem \ref{theorem:sw_main}) is thus
	\begin{align}
		2\nonumber&\sum_{t\in[T]}L \sum^{t-1}_{s = 1\vee (t-w)}\left\|\theta_s-\theta_{s+1}\right\|+\beta\left\|X_t\right\|_{V^{-1}_{t-1}} \\
		=&\widetilde{O}\left(wB_T+\frac{dT}{\sqrt{w}}\right).
	\end{align} 
	If $B_T$ is known, the learner can set $w=\lfloor d^{2/3}T^{2/3}B_T^{-2/3}\rfloor$ and achieve a regret upper bound $\widetilde{O} (d^{2/3}B_T^{1/3}T^{2/3} ).$ If $B_T$ is not known, which is often the case in practice, the learner can set $w=\lfloor (dT)^{2/3}\rfloor$ to obtain a regret upper bound $\widetilde{O}(d^{2/3}(B_T+1)T^{2/3}).$
	\subsection{Design Details}
	\label{sec:swofu_details}
	In this section, we describe the details of the \swofu. Following its design guideline, the \swofu~selects a positive regularization parameter $\lambda~(>0),$ and initializes $V_0=\lambda I.$ In each round $t,$ the \swofu~first computes the estimate $\hat{\theta}_t$ for $\theta_t$ according to eq. \ref{eq:sw20}, and then finds the action $X_t$ with largest UCB by solving the optimization problem (\ref{eq:sw_policy}). Afterwards, the corresponding reward $Y_t$ is observed. The pseudo-code of the \swofu~is shown in Algorithm \ref{alg:swofu}.
	\begin{algorithm}[!ht]
		\caption{\swofu}
		\label{alg:swofu}
		\begin{algorithmic}[1]
			\State \textbf{Input:} Sliding window size $w$, dimension $d,$ variance proxy of the noise terms $R,$ upper bound of all the actions' $\ell_2$ norms $L,$ upper bound of all the $\theta_t$'s $\ell_2$ norms $S,$ and regularization constant $\lambda.$
			\State \textbf{Initialization:} $V_0\leftarrow\lambda I.$
			\For{$t=1,\ldots,T$}
			\State{$\hat{\theta}_t\leftarrow V_{t-1}^{-1}\left(\sum_{s=1\vee(t-w)}^{t-1}X_sY_s\right).$}
			\State{$X_t\leftarrow\argmax_{x\in D_t}\left\{x^\top\hat{\theta}_t\right.$}
			\State{$\quad\left.+ \left\|x\right\|_{V^{-1}_{t-1}} \left[ R\sqrt{d\ln\left(\frac{1+wL^2/\lambda}{\delta}\right)}+ \sqrt{\lambda}S \right]\right\}.$}
			\State{$Y_t\leftarrow\langle X_t,\theta_t\rangle+\eta_t.$}
			\State{$V_t\leftarrow\lambda I+\sum_{s=1\vee(t-w+1)}^tX_sX_s^{\top}.$}
			\EndFor
		\end{algorithmic}
	\end{algorithm}
	\subsection{Regret Analysis}
	\label{sec:swofu_regret}
	We are now ready to formally state a regret upper bound of the \swofu.
	\begin{theorem}
		\label{theorem:sw_main}
		The dynamic regret of the \swofu~is upper bounded as
		$
			\R_T\left(\swofu\right)=\widetilde{O}\left(wB_T+\frac{dT}{\sqrt{w}}\right).
	$
		When $B_t~(>0)$ is known, by taking $w=O((dT)^{2/3}B_T^{-2/3}),$ the dynamic regret of the \swofu~is
	$
			\R_T\left(\swofu\right)=\widetilde{O}\left(d^{\frac{2}{3}}B_T^{\frac{1}{3}}T^{\frac{2}{3}}\right).
	$
		When $B_t$ is unknown, by taking $w=O((dT)^{2/3}),$ the dynamic regret of the \swofu~is
	$
			\R_T\left(\swofu\right)=\widetilde{O}\left(d^{\frac{2}{3}}\left(B_T+1\right)T^{\frac{2}{3}}\right).
	$
	\end{theorem}
	\begin{proof}[Sketch Proof.]
		The proof utilizes the fact that the per round regret of the \swofu~is upper bounded by the UCB of the chosen action, and decomposes the UCB into two separated terms according to Lemmas \ref{lemma:sw} and \ref{lemma:sw1}, \ie,
		\begin{align*}
			&\text{regret in round }t=\text{regret due to non-stationarity in}\\
			&\text{round }t+\text{regret due to estimation error in round }t.
		\end{align*}
		The first term can be upper bounded by a intuitive telescoping sum; while for the second term, although a similar quantity is analyzed by the authors of \cite{AYPS11} using a (beautiful) matrix telescoping technique under the stationary environment, we note that due to the ``forgetting principle" of the \swofu, we cannot directly adopt the technique. Our proof thus makes a novel use of the Sherman-Morrison formula to overcome the barrier. Please refer to Section \ref{sec:theorem:sw_main} of appendix for the complete proof.
	\end{proof}
	
	\section{Bandit-over-Bandit (\texttt{BOB}) Algorithm: Automatically Adapting to the Unknown Variation Budget}
	\label{sec:bob}
	In Section \ref{sec:swofu}, we have seen that, by properly tuning $w,$ the learner can achieve a first order optimal $\widetilde{O}\left(d^{2/3}(B_T+1)T^{2/3}\right)$ regret bound even if the knowledge of $B_T$ is not available. However, in the case of an unknown and large $B_T$, \ie, $B_T=\Omega(T^{1/3}),$ the bound becomes meaningless as it is linear in $T.$ To handle this case, we wish to design an online algorithm that incurs a dynamic regret of order $\widetilde{O}\left(d^{\nu}B_T^{1-\sigma}T^{\sigma}\right)$ for some $\nu\in[0,1]$ and $\sigma\in(0,1)$, without knowing $B_T$. Note from Theorem \ref{theorem:lower_bound}, no algorithm can achieve a dynamic regret of order $o(d^{2/3}B_T^{1/3}T^{2/3})$, so we must have $\sigma\geq\frac{2}{3}$. 
	In this section, we develop a novel Bandit-over-Bandit (\texttt{BOB}) algorithm that achieves a regret of $\tilde{O}(d^{2/3}B_T^{1/4}T^{3/4})$. Hence, (\texttt{BOB}) still has a dynamic regret sublinear in $T$ when $B_T = \Theta(T^{\rho})$ for any $\rho\in (0, 1)$ and $B_T$ is not known, unlike the \swofu .   
	\subsection{Design Challenges}
	\label{sec:bob_challenges}
	Reviewing Theorem \ref{theorem:sw_main}, we know that setting the window length $w$ to a fixed value \begin{align}\label{eq:w_opt}w^*=\left\lfloor (dT)^{2/3}(B_T+1)^{-2/3}\right\rfloor\end{align}
	can give us a $\widetilde{O}\left(d^{2/3}(B_T+1)^{1/3}T^{2/3}\right)$ regret bound. But when $B_T$ is not provided a priori, we need to also ``learn" the unknown $B_T$ in order to properly tune $w.$ In a more restrictive setting in which the differences between consecutive $\theta_t$'s follow some underlying stochastic process, one possible approach is applying a suitable machine learning technique to learn the underlying stochastic process at the beginning, and tune the parameter $w$ accordingly. In the more general setting, however, this strategy cannot work as the change between consecutive $\theta_t$'s can be arbitrary (or even adversarially) as long as the total variation is bounded by $B_T.$ 
	\subsection{Design Intuition}
	\label{sec:bob_intuition}
	The above mentioned observations as well as the established results motivate us to make use of the \swofu~as a sub-routine, and ``hedge" against the changes of $\theta_t$'s to identify a reasonable fixed window length \cite{ABFS02}. To this end, we describe the main idea of the Bandit-over-Bandit (\texttt{BOB}) algorithm. The \bob~divides the whole time horizon into $\lceil T/H\rceil$ blocks of equal length $H$ rounds (the last block can possibly have less than $H$ rounds), and specifies a set $J~\left(\subseteq[H]\right)$ from which each $w_i$ is drawn from. For each block $i\in\left[\lceil T/H\rceil\right]$, the \bob~first selects a window length $w_i~(\in J)$, and initiates a new copy of the \swofu~with the selected window length as a sub-routine to choose actions for this block. On top of this, the \bob~also maintains a separate algorithm for adversarial multi-armed bandits, \eg, the EXP3 algorithm, to govern the selection of window length for each block, and thus the name Bandit-over-Bandit. Here, the total reward of each block is used as feedback for the EXP3 algorithm.
	
	To determine $H$ and $J$, we first consider the regret of the \bob. Since the window length is constrained to be in $J,$ and is less than or equal to $H,$ $w^*$ is not necessarily the optimal window length in this case, and we hence denote the optimally tuned window length as $w^{\dag}.$ By design of the \bob, its regret can be decomposed as the regret of an algorithm that optimally tunes the window length $w_i=w^{\dag}$ for each block $i$ plus the loss due to learning the value $w^{\dag}$ with the EXP3 algorithm,
	\begin{align}
		\label{eq:bob_decompose}
		\nonumber&\Ex\left[\text{Regret}_T(\bob)\right]\\
		\nonumber=&\Ex\left[\sum_{t=1}^T\langle x_t^*,\theta_t\rangle-\sum_{t=1}^T\langle X_t,\theta_t\rangle\right]\\
		\nonumber=&\Ex\left[\sum_{t=1}^T\langle x_t^*,\theta_t\rangle-\sum_{i=1}^{\lceil T/H\rceil}\sum_{t=(i-1)H+1}^{i\cdot H\wedge T}\left\langle X_t\left(w^{\dag}\right),\theta_t\right\rangle\right]\\
		\nonumber&+\Ex\left[\sum_{i=1}^{\lceil T/H\rceil}\sum_{t=(i-1)H+1}^{i\cdot H\wedge T}\left\langle X_t\left(w^{\dag}\right),\theta_t\right\rangle\right.\\
		&~\left.-\sum_{i=1}^{\lceil T/H\rceil}\sum_{t=(i-1)H+1}^{i\cdot H\wedge T}\left\langle X_t\left(w_i\right),\theta_t\right\rangle\right].
	\end{align}
	Here for a round $t$ in block $i,$ $X_t(w)$ refers to the action selected in round $t$ by the \swofu~with window length $w\wedge (t-(i-1)H-1)$ initiated at the beginning of block $i.$ 
	
	By Theorem \ref{theorem:sw_main}, the first expectation in eq. (\ref{eq:bob_decompose}) can be upper bounded as
	\begin{align}
		\label{eq:bob_decompose1}
		&\nonumber\Ex\left[\sum_{t=1}^T\langle x_t^*,\theta_t\rangle-\sum_{i=1}^{\lceil T/H\rceil}\sum_{t=(i-1)H+1}^{i\cdot H\wedge T}\left\langle X_t\left(w^{\dag}\right),\theta_t\right\rangle\right]\\
		\nonumber=&\Ex\left[\sum_{i=1}^{\lceil T/H\rceil}\sum_{t=(i-1)H+1}^{i\cdot H\wedge T}\left\langle x^*_t-X_t\left(w^{\dag}\right),\theta_t\right\rangle\right]\\
		\nonumber=&\sum_{i=1}^{\lceil T/H\rceil}\widetilde{O}\left(w^{\dag}B_T(i)+\frac{dH}{\sqrt{w^{\dag}}}\right)\\
		=&\widetilde{O}\left(w^{\dag}B_T+\frac{dT}{\sqrt{w^{\dag}}}\right),
	\end{align}
	where $B_{T}(i)=\sum_{t=(i-1)H+1}^{(i\cdot H\wedge t)-1}\|\theta_{t}-\theta_{t+1}\|$ is the total variation in block $i.$ 
	
	We then turn to the second expectation in eq. (\ref{eq:bob_decompose}). We can easily see that the number of rounds for the EXP3 algorithm is $\lceil T/H\rceil$ and the number of possible values of $w_i$'s is $|J|.$ Denoting the maximum absolute sum of rewards of any block as  random variable $Q,$ the authors of \cite{ABFS02} gives the following regret bound.
	\begin{align}
		\label{eq:bob_decompose2}
		\nonumber&\Ex\left[\sum_{i=1}^{\lceil T/H\rceil}\sum_{t=(i-1)H+1}^{i\cdot H\wedge T}\left\langle X_t\left(w^{\dag}\right),\theta_t\right\rangle\right.\\
		\nonumber&~\left.-\sum_{i=1}^{\lceil T/H\rceil}\sum_{t=(i-1)H+1}^{i\cdot H\wedge T}\left\langle X_t\left(w_i\right),\theta_t\right\rangle\right]\\
		\leq&\Ex\left[\widetilde{O}\left(Q\sqrt{\frac{|J|T}{H}}\right)\right].
	\end{align} 
	To proceed, we have to give a high probability upper bound for $Q.$
	\begin{lemma}
		\label{lemma:bob}
	$
			\Pr\left(Q\leq H+2R\sqrt{H\ln\frac{T}{\sqrt{H}}}\right)\geq 1-\frac{2}{T}.
	$
	\end{lemma}
	\begin{proof}[Sketch Proof]
		The proof makes use of the $R$-sub-Gaussian property of the noise terms as well as the union bound over all the blocks. Please refer to Section \ref{sec:lemma:bob} of the appendix for the complete proof.
	\end{proof}
	Note that the regret of our problem is at most $T,$ eq. (\ref{eq:bob_decompose2}) can be further upper bounded as
	\begin{align}
		\label{eq:bob_decompose3}
		\nonumber&\Ex\left[\sum_{i=1}^{\lceil T/H\rceil}\sum_{t=(i-1)H+1}^{i\cdot H\wedge T}\left\langle X_t\left(w^{\dag}\right),\theta_t\right\rangle\right.\\
		\nonumber&~\left.-\sum_{i=1}^{\lceil T/H\rceil}\sum_{t=(i-1)H+1}^{i\cdot H\wedge T}\left\langle X_t\left(w_i\right),\theta_t\right\rangle\right]\\
		\nonumber\leq&\Ex\left[\widetilde{O}\left(Q\sqrt{\frac{|J|T}{H}}\right)\middle| Q\leq H+2HR\sqrt{\ln T}\right]\\
		\nonumber&~\times\Pr\left(Q\leq H+2HR\sqrt{\ln T}\right)\\
		\nonumber&+\Ex\left[\widetilde{O}\left(Q\sqrt{\frac{|J|T}{H}}\right)\middle| Q\geq H+2HR\sqrt{\ln T}\right]\\
		\nonumber&~\times\Pr\left(Q\geq H+2HR\sqrt{\ln T}\right)\\
		\nonumber=&\widetilde{O}\left(\sqrt{H|J|T}\right)+T\cdot\frac{2}{T}\\
		=&\widetilde{O}\left(\sqrt{H|J|T}\right).
	\end{align} 
	Combining eq. (\ref{eq:bob_decompose}), (\ref{eq:bob_decompose1}), and (\ref{eq:bob_decompose3}), the regret of the \bob~is 
	\begin{align}
		\label{eq:bob_regret}&\R_T(\bob)=\widetilde{O}\left(w^{\dag}B_T+\frac{dT}{\sqrt{w^{\dag}}}+\sqrt{H|J|T}\right).
	\end{align}
	Eq. (\ref{eq:bob_regret}) exhibits a similar structure to the regret of the \swofu~as stated in Theorem \ref{theorem:sw_main}, and this immediately indicates a clear trade-off in the design of the block length $H.$ On one hand, $H$ should be small to control the regret incurred by the EXP3 algorithm in identifying $w^{\dag},$ \ie, the third term in eq. (\ref{eq:bob_regret}); on the other hand, $H$ should also be large enough so that $w^{\dag}$ can get close to $w^*=\lfloor (dT)^{2/3}(B_T+1)^{-2/3}\rfloor$ so that the sum of the first two terms in eq. (\ref{eq:bob_regret}) is minimized. A more careful inspection also reveals the tension in the design of $J.$ Obviously, we hope that $|J|$ is small, but we also wish $J$ to be dense enough so that it forms a cover to the set $H.$ Otherwise, even if $H$ is large and $w^{\dag}$ can approach $w^*,$ approximating $w^*$ with any element in $J$ can cause a major loss.
	
	These observations suggest the following choice of $J.$ 
	\begin{align}
		J=\left\{H^0,\left\lfloor H^{\frac{1}{\Delta}}\right\rfloor,\ldots, H\right\}
	\end{align}
	for some positive integer $\Delta.$ For the purpose of analysis, suppose the (unknown) parameter $w^{\dag}$ can be expressed as $\text{clip}_J\left(\lfloor d^{\epsilon}T^{\alpha}(B_T+1)^{-\alpha}\rfloor\right)$ with some $\alpha\in[0,1]$ and $\epsilon>0$ to be determined, where $\text{clip}_J(x)$ finds the largest element in $J$ that does not exceed $x.$ Notice that $|J|=\Delta+1,$ the regret of the \bob~then becomes
	\begin{align}
		&\nonumber\R_T(\bob)\\
		\nonumber=&\widetilde{O}\left(d^{\epsilon}\left(B_T+1\right)^{1-\alpha}T^{\alpha}H^{\frac{2}{\Delta}}\right.\\
		\nonumber&~\left.+d^{1-\frac{\epsilon}{2}}\left(B_T+1\right)^{\frac{\alpha}{2}}T^{1-\frac{\alpha}{2}}H^{\frac{2}{\Delta}} +\sqrt{HT\Delta }\right)\\
		=&\widetilde{O}\left(d^{\epsilon}\left(B_T+1\right)^{1-\alpha}T^{\alpha}\right. \nonumber\\
		\label{eq:bob_regret1}&~\left. +d^{1-\frac{\epsilon}{2}}\left(B_T+1\right)^{\frac{\alpha}{2}}T^{1-\frac{\alpha}{2}}+\sqrt{HT}\right),
	\end{align}
	where we have set $\Delta=\lceil\ln H\rceil$ in eq. (\ref{eq:bob_regret1}); Since $w^{\dag}\in J$ (or $w^{\dag}\leq H$), and $H$ should not depend on $B_T,$ we can set 
	\begin{align}H=\left\lfloor d^{\epsilon}T^{\alpha}\right\rfloor,\end{align} and the regret of the \bob~(to be formalized in Theorem \ref{theorem:bob}) is upper bounded as
	\begin{align}
		&\nonumber\R_T(\bob)\\
		\nonumber=&\widetilde{O}\left(d^{\epsilon}\left(B_T+1\right)^{1-\alpha}T^{\alpha}+d^{1-\frac{\epsilon}{2}}\left(B_T+1\right)^{\frac{\alpha}{2}}T^{1-\frac{\alpha}{2}}\right.\\
		\nonumber&\left.+d^{\frac{\epsilon}{2}}T^{\frac{1}{2}+\frac{\alpha}{2}}\right)\\
		=&\widetilde{O}\left(d^{\frac{2}{3}}\left(B_T+1\right)^{\frac{1}{4}}T^{\frac{3}{4}}\right).
	\end{align}
	Here, we have taken $\alpha=1/2$ and $\epsilon=2/3,$ but we have to emphasize that the choice of $w^{\dag},\alpha,$ and $\epsilon$ are purely for an analysis purpose. The only parameters that we need to design are 
	\begin{align}
		\label{eq:bob_parameters}
		&H=\left\lfloor d^{\frac{2}{3}}T^{\frac{1}{2}}\right\rfloor,\Delta=\lceil\ln H\rceil,J=\left\{1,\left\lfloor H^{\frac{1}{\Delta}}\right\rfloor,\ldots, H\right\},
	\end{align}
	which clearly do not depend on $B_T.$
	
	\subsection{Design Details}
	\label{sec:bob_details}
	We are now ready to describe the details of the \bob. With $H,\Delta$ and $J$ defined as eq. (\ref{eq:bob_parameters}), the \bob~additionally initiates the parameter 
	\begin{align}
		\label{eq:bob_parameters1}
		\nonumber&\gamma=\min\left\{1,\sqrt{\frac{(\Delta+1)\ln(\Delta+1)}{(e-1)\lceil T/H\rceil}}\right\},\\
		&s_{j,1}=1\quad\forall j=0,1,\ldots,\Delta.
	\end{align}
	for the EXP3 algorithm \cite{ABFS02}. The \bob~then divides the time horizon $T$ into $\lceil T/H\rceil$ blocks of length $H$ rounds (except for the last block, which can be less than $H$ rounds). At the beginning of each block $i\in\left[\lceil T/H\rceil\right],$ the \bob~first sets
	\begin{align}\label{eq:p_j_i}
		p_{j,i}=(1-\gamma)\frac{s_{j,i}}{\sum_{u=0}^{\Delta}s_{u,i}}+\frac{\gamma}{\Delta+1}\quad\forall j=0,1,\ldots,\Delta,
	\end{align}
	and then sets 
	$j_i=j\quad\text{with probability }p_{j,i}$ for all $j=0,\ldots,\Delta.$
	The selected window length is thus $w_i=\left\lfloor H^{j_i/\Delta}\right\rfloor.$ Afterwards, the \bob~selects actions $X_t$ by running the \swofu~with window length $w_i$ for each round $t$ in block $i,$ and the total collected reward is $
		\sum_{t=(i-1)H+1}^{i\cdot H\wedge T}Y_t=\sum_{t=(i-1)H+1}^{i\cdot H\wedge T}\langle X_t,\theta_t\rangle+\eta_t.$
	Finally, the rewards are rescaled by dividing $2H+4R\sqrt{H\ln (T/\sqrt{H})},$ and then added by $1/2$ so that it lies within $[0,1]$ with high probability, and the parameter $s_{j_i,i+1}$ is set to
	\begin{align}\label{eq:s_j_i}
		s_{j_i,i}\cdot\exp\left(\frac{\gamma}{(\Delta+1)p_{j_i,i}}\left(\frac{1}{2}+\frac{\sum_{t=(i-1)H+1}^{i\cdot H\wedge T}Y_t}{2H+4R\sqrt{H\ln\frac{T}{\sqrt{H}}})}\right)\right);
	\end{align}
	while $s_{u,i+1}$ is the same as $s_{u,i}$ for all $u\neq j_i.$ The pseudo-code of the \bob~is shown in Algorithm \ref{alg:bob}.
	\begin{algorithm}[!ht]
		\caption{\bob}
		\label{alg:bob}
		\begin{algorithmic}[1]
			\State \textbf{Input:} Time horizon $T$, the dimension $d,$ variance proxy of the noise terms $R,$ upper bound of all the actions' $\ell_2$ norms $L,$ upper bound of all the $\theta_t$'s $\ell_2$ norms $S,$ and a constant $\lambda.$
			\State \textbf{Initialize} $H, \Delta, J$ by eq. (\ref{eq:bob_parameters}), $\gamma, \{s_{j, 1}\}^{\Delta}_{j=0}$ by eq. (\ref{eq:bob_parameters1}).		
			\For{$i=1,2,\ldots,\lceil T/H\rceil$}
			\State Define distribution $(p_{j, i})^{\Delta}_{j=0}$ by eq. (\ref{eq:p_j_i}).		
			\State{Set $j_t\leftarrow j$ with probability $p_{j,i}$.}
			\State{$w_i\leftarrow\left\lfloor H^{j_t/\Delta}\right\rfloor.$}
			\State{$V_{(i-1)H}=\lambda I.$}
			\For{$t=(i-1)H+1,\ldots,i\cdot H\wedge T$}
			\State{$\hat{\theta}_t\leftarrow V_{t-1}^{-1}\left(\sum_{s=[(i-1)H+1]\vee(t-w_i)}^{t-1}X_sY_s\right).$}
			\State{Pull arm $X_t\leftarrow\argmax_{x\in D_t}\left\{x^\top\hat{\theta}_t \right.$}
			\Statex{$~\left.+ \left\|x\right\|_{V^{-1}_{t-1}} \left[ R\sqrt{d\ln\left(T\left(1+w_iL^2/\lambda\right)\right)}+ \sqrt{\lambda}S \right]\right\}.$}
			\State{Observe $Y_t = \langle X_t,\theta_t\rangle+\eta_t.$}
			\State{$V_t\leftarrow \lambda I+\sum_{s=[(i-1)H+1]\vee(t+1-w_i)}^{t}X_sX_s^{\top}.$}
			\EndFor
			\State{Define $s_{j_i,i+1}$ according to eq. (\ref{eq:s_j_i}})
			\State{Define $s_{u,i+1}\leftarrow s_{u,i}~\forall u\neq j_i$}
			\EndFor
		\end{algorithmic}
	\end{algorithm}
	\subsection{Regret Analysis}
	\label{sec:bob_regret}
	We are now ready to present the regret analysis of the \bob.
	\begin{theorem}
		\label{theorem:bob}
		The dynamic regret of the \bob~with the \swofu~as a sub-routine is
	$
			\R_T\left(\bob\right)=\widetilde{O}\left(d^{\frac{2}{3}}\left(B_T+1\right)^{\frac{1}{4}}T^{\frac{3}{4}}\right).
	$
	\end{theorem}
	\begin{proof}[Sketch Proof]
		The proof of the theorem essentially follows Section \ref{sec:bob_intuition}, and we thus omit it.
	\end{proof}

\section{Numerical Experiments}
As a complement to our theoretical results, we conduct numerical experiments on synthetic data to compare the regret performances of the \swofu~and the \bob~with a modified EXP3.S algorithm analyzed in \cite{BGZ18}. Note that the algorithms in \cite{BGZ18} are designed for the stochastic MAB setting, a special case of us, we follow the setup of \cite{BGZ18} for fair comparisons. Specifically, we consider a 2-armed bandit setting, and we vary $T$ from $3\times10^4$ to $2.4\times 10^5$ with a step size of $3\times10^4.$ We set $\theta_t$ to be the following sinusoidal process, \ie, $\forall t\in[T],$
$
\theta_t=\begin{pmatrix}
		0.5+0.3\sin\left({5B_T\pi t}/{T}\right)\\
		0.5+0.3\sin\left(\pi+{5B_T\pi t}/{T}\right)
	\end{pmatrix}.
$
The total variation of the $\theta_t$'s across the whole time horizon is upper bounded by $\sqrt{2}B_T=O(B_T).$ We also use i.i.d. normal distribution with $R=0.1$ for the noise terms. 
\paragraph{Known Constant Variation Budget.} We start from the known constant variation budget case, \ie, $B_T=1,$ to measure the regret growth of the two optimal algorithms, \ie,  the \swofu~and the modified EXP3.S algorithm, with respect to the total number of rounds. The log-log plot is shown in Fig. \ref{fig:const_var}. From the plot, we can see that the regret of \swofu~is only about $20\%$ of the regret of EXP3.S algorithm.
\begin{figure}[!ht]
	\vspace{-4mm}
	\centering
	\includegraphics[width=7cm,height=5cm]{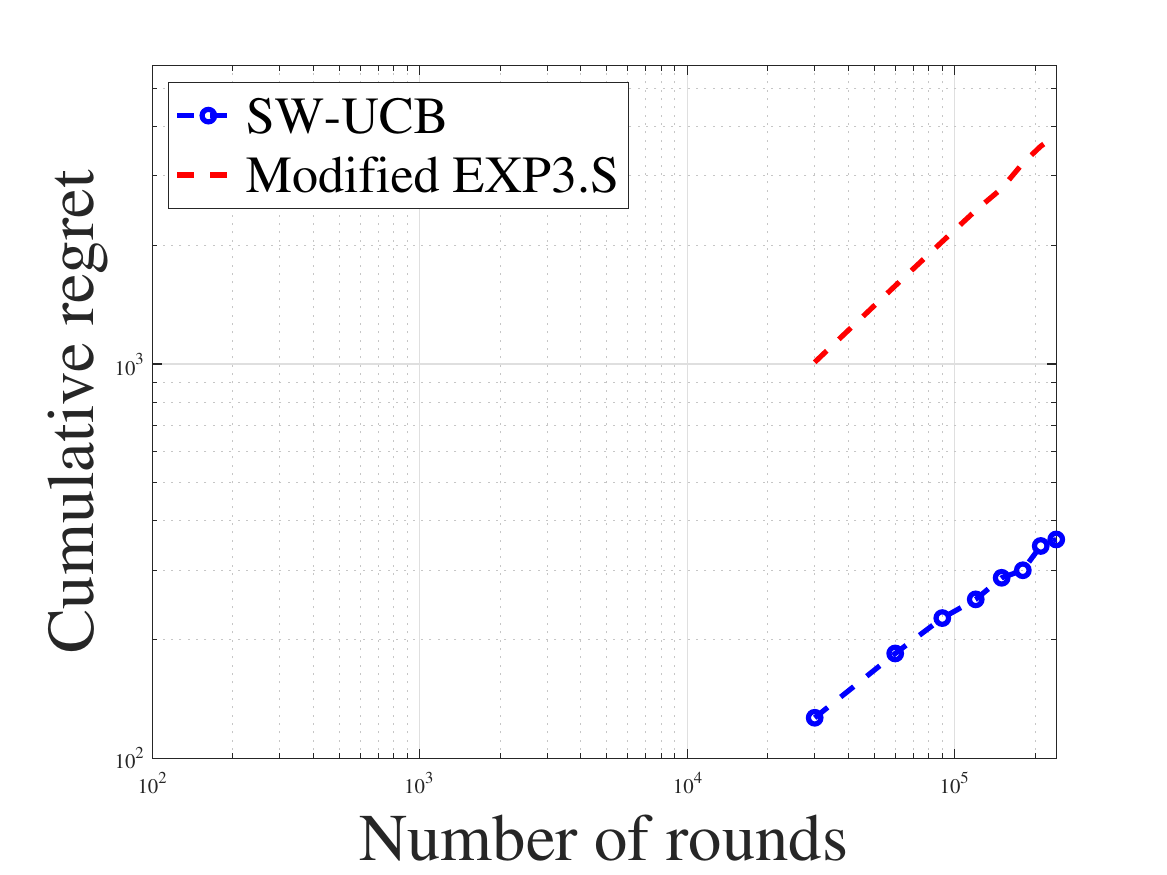}
		\vspace{-3mm}
	\caption{Log-log plot for $B_T=O(1).$}
	\vspace{-6mm}
	\label{fig:const_var}
\end{figure}
\paragraph{Unknown Time-Dependent Variation Budget.} We then turn to the more realistic time-dependent variation budget case, \ie, $B_T=T^{1/3}.$ As the modified EXP3.S algorithm does not apply to this setting, we compare the performances of the \swofu~and the \bob. The log-log plot is shown in Fig. \ref{fig:inc_var}. From the results, we verify that the slope of the regret growth of both algorithms roughly match the established results, and the regret of \bob's~is much smaller than that of the \swofu's.
\begin{figure}[!ht]
	\vspace{-4mm}
	\centering
	\includegraphics[width=7cm,height=5cm]{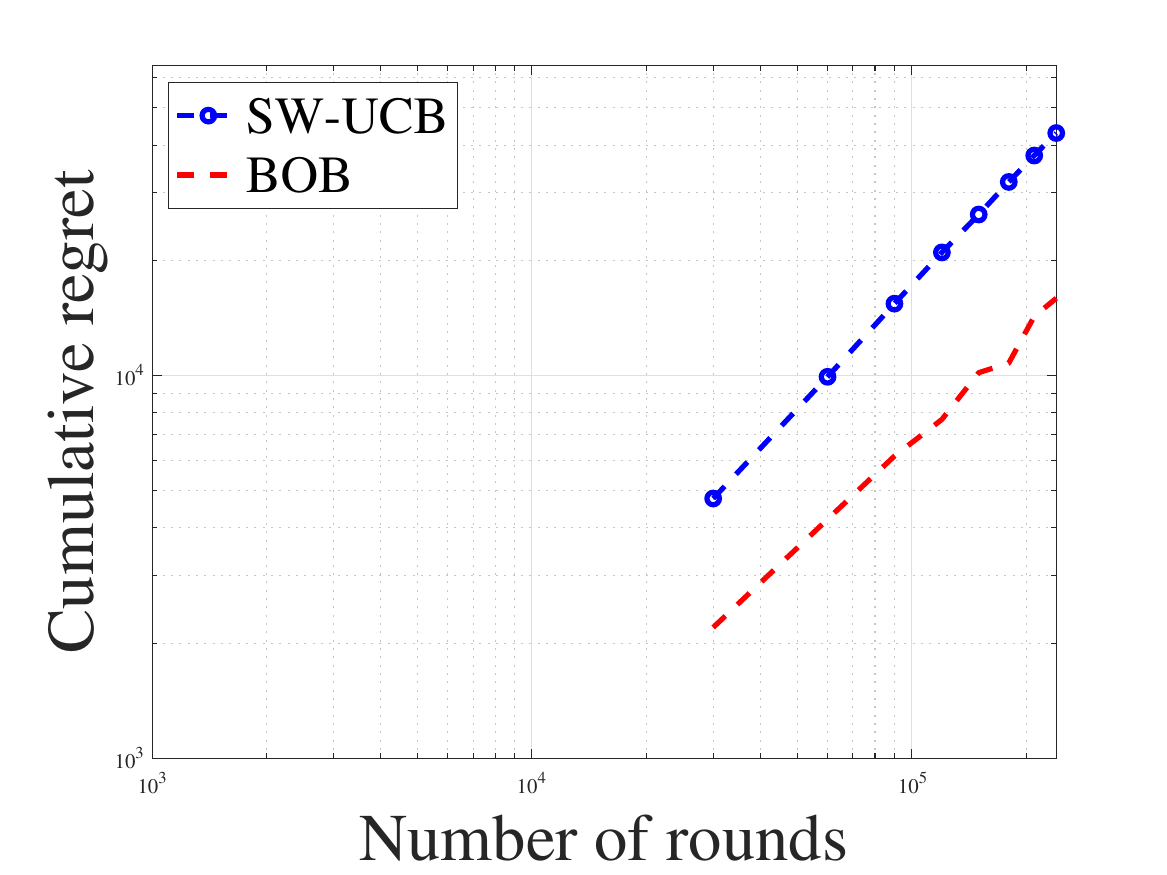}
		\vspace{-3mm}
	\caption{Log-log plot for $B_T=O(T^{1/3}).$}
	\vspace{-5mm}
	\label{fig:inc_var}
\end{figure}

\bibliographystyle{abbrv}
\bibliography{paperlist}

\begin{thebibliography}{10}

\bibitem{AYPS11}
Y.~Abbasi-Yadkori, D.~P\'{a}l, and C.~Szepesv\'{a}ri.
\newblock Improved algorithms for linear stochastic bandits.
\newblock In {\em Proceedings of the 24th Annual Conference on Neural
  Information Processing Systems (NIPS)}, 2011.

\bibitem{AL17}
M.~Abeille and A.~Lazaric.
\newblock Linear thompson sampling revisited.
\newblock In {\em Proceedings of the 20th International Conference on
  Artificial Intelligence and Statistics (AISTATS)}, 2017.

\bibitem{ALNS17}
A.~Agarwal, H.~Luo, B.~Neyshabur, and R.~E. Schapire.
\newblock Corralling a band of bandit algorithms.
\newblock In {\em Proceedings of the 30th Annual Conference on Learning Theory
  (COLT)}, 2017.

\bibitem{AG13}
S.~Agrawal and N.~Goyal.
\newblock Thompson sampling for contextual bandits with linear payoffs.
\newblock In {\em Proceedings of the 30th International Conference on Machine
  Learning (ICML)}, 2013.

\bibitem{A02}
P.~Auer.
\newblock Using confidence bounds for exploitation-exploration trade-offs.
\newblock In {\em Journal of Machine Learning Research, 3:397--422, 2002.},
  2002.

\bibitem{ABF02}
P.~Auer, N.~Cesa-Bianchi, and P.~Fischer.
\newblock Finite-time analysis of the multiarmed bandit problem.
\newblock {\em Machine learning, 47, 235--256}, 2002.

\bibitem{ABFS02}
P.~Auer, N.~Cesa-Bianchi, Y.~Freund, and R.~Schapire.
\newblock The nonstochastic multiarmed bandit problem.
\newblock In {\em SIAM Journal on Computing, 2002, Vol. 32, No. 1 : pp.
  48--77}, 2002.

\bibitem{BGZ18}
O.~Besbes, Y.~Gur, and A.~Zeevi.
\newblock Optimal exploration-exploitation in a multi-armed-bandit problem with
  non-stationary rewards.
\newblock In {\em Available at: https://ssrn.com/abstract=2436629, year =
  {2018},}.

\bibitem{BGZ14}
O.~Besbes, Y.~Gur, and A.~Zeevi.
\newblock Stochastic multi-armed bandit with non-stationary rewards.
\newblock In {\em Proceedings of the 27th Annual Conference on Neural
  Information Processing Systems (NIPS)}, 2014.

\bibitem{BGZ15}
O.~Besbes, Y.~Gur, and A.~Zeevi.
\newblock Non-stationary stochastic optimization.
\newblock In {\em Operations Research, 2015, 63 (5), 1227--1244}, 2015.

\bibitem{BC12}
S.~Bubeck and N.~Cesa{-}Bianchi.
\newblock {\em Regret Analysis of Stochastic and Nonstochastic Multi-armed
  Bandit Problems}.
\newblock Foundations and Trends in Machine Learning, 2012, Vol. 5, No. 1: pp.
  1--122, 2012.

\bibitem{CBL06}
N.~Cesa-Bianchi and G.~Lugosi.
\newblock {\em Prediction, Learning, and Games}.
\newblock Cambridge University Press, 2006.

\bibitem{CYLMLJZ12}
C.~Chiang, T.~Yang, C.~Lee, M.~Mahdavi, C.~Lu, R.~Jin, and S.~Zhu.
\newblock Online optimization with gradual variations.
\newblock In {\em Proceedings of the 25th Conference on Learning Theory
  (COLT)}, 2012.

\bibitem{CLRS11}
W.~Chu, L.~Li, L.~Reyzin, and R.~Schapire.
\newblock Contextual bandits with linear payoff functions.
\newblock In {\em Proceedings of the the 14th International Conference on
  Artificial Intelligence and Statistics (AISTATS)}, 2011.

\bibitem{DHK08}
V.~Dani, T.~Hayes, and S.~Kakade.
\newblock Stochastic linear optimization under bandit feedback.
\newblock In {\em Proceedings of the 21st Conference on Learning Theory
  (COLT)}, 2008.

\bibitem{FauryRAC21}
L.~Faury, Y.~Russac, M.~Abeille, and C.~Calauzenes.
\newblock Regret bounds for generalized linear bandits under parameter drift.
\newblock In {\em https://arxiv.org/abs/2103.05750}, 2021.

\bibitem{GM11}
A.~Garivier and E.~Moulines.
\newblock On upper-confidence bound policies for switching bandit problems.
\newblock In {\em The 22nd International Conferenc on Algorithmic Learning
  Theory (ALT)}, 2011.

\bibitem{JRSS15}
A.~Jadbabaie, A.~Rakhlin, S.~Shahrampour, and K.~Sridharan.
\newblock Online optimization : Competing with dynamic comparators.
\newblock In {\em Proceedings of the 18th International Conference on
  Artificial Intelligence and Statistics (AISTATS)}, 2015.

\bibitem{KA16}
Z.~Karnin and O.~Anava.
\newblock Multi-armed bandits: Competing with optimal sequences.
\newblock In {\em Procedding of the 29th Annual Conference on Neural
  Information Processing Systems (NIPS)}, 2016.

\bibitem{KZ16}
N.~Keskin and A.~Zeevi.
\newblock Chasing demand: Learning and earning in a changing environments.
\newblock In {\em Mathematics of Operations Research, 2016, 42(2), 277--307},
  2016.

\bibitem{LS18}
T.~Lattimore and C.~Szepesv\'{a}ri.
\newblock {\em Bandit Algorithms}.
\newblock Cambridge University Press 2018, 2018.

\bibitem{LWAL18}
H.~Luo, C.~Wei, A.~Agarwal, and J.~Langford.
\newblock Efficient contextual bandits in non-stationary worlds.
\newblock In {\em Proceedings of the 31st Conference on Learning Theory
  (COLT)}, 2018.

\bibitem{RH18}
R.~Rigollet and J.~H\"utter.
\newblock {\em High Dimensional Statistics}.
\newblock Lecture Notes, 2018, 2018.

\bibitem{RT10}
P.~Rusmevichientong and J.~Tsitsiklis.
\newblock Linearly parameterized bandits.
\newblock In {\em Mathematics of Operations Research, 35(2):395--411, 2010},
  2010.

\bibitem{RVR14}
D.~Russo and B.~V. Roy.
\newblock Learning to optimize via posterior sampling.
\newblock In {\em Mathematics of Operations Research}, 2014.

\bibitem{WHL16}
C.-Y. Wei, Y.-T. Hong, and C.-J. Lu.
\newblock Tracking the best expert in non-stationary stochastic environments.
\newblock In {\em Proceedings of the 29th Annual Conference on Neural
  Information Processing (NIPS)}, 2016.

\end{thebibliography}
\newpage

\onecolumn
\appendix

	\section{Proof of Theorem \ref{theorem:lower_bound}}
	\label{sec:theorem:lower_bound}
	First, let's review the lower bound of the linear bandit setting. The linear bandit setting is almost identical to ours except that the $\theta_t$'s do not vary across rounds, and are equal to the same (unknown) $\theta,$ \ie, $\forall t\in[T]~\theta_t=\theta.$
	\begin{lemma}[\cite{LS18}]
		\label{lemma:lower_bound}
		For any $T_0\geq\sqrt{d}/2$ and let $D=\left\{x\in\Re^d:\|x\|\leq1\right\},$ then there exists a $\theta\in\left\{\pm\sqrt{d/4T_0}\right\}^d,$ such that the worst case regret of any algorithm for linear bandits with unknown parameter $\theta$ is $\Omega(d\sqrt{T_0}).$
	\end{lemma}
	Going back to the non-stationary environment, suppose nature divides the whole time horizon into $\lceil T/H\rceil$ blocks of equal length $H$ rounds (the last block can possibly have less than $H$ rounds), and each block is a decoupled linear bandit instance so that the knowledge of previous blocks cannot help the decision within the current block. Following Lemma \ref{lemma:lower_bound}, we restrict the sequence of $\theta_t$'s are drawn from the set $\left\{\pm\sqrt{d/4H}\right\}^d.$ Moreover, $\theta_t$'s remain fixed within a block, and can vary across different blocks, \ie,
	\begin{align}
		\forall i\in\left[\left\lceil \frac{T}{H}\right\rceil\right]\forall t_1,t_2\in[(i-1)H+1,i\cdot H\wedge T]\quad \theta_{t_1}=\theta_{t_2}.
	\end{align}
	We argue that even if the learner knows this additional information, it still incur a regret $\Omega(d^{2/3}B_T^{1/3}T^{2/3}).$ Note that different blocks are completely decoupled, and information is thus not passed across blocks. Therefore, the regret of each block is $\Omega\left(d\sqrt{H}\right),$ and the total regret is at least
	\begin{align}
		\left(\left\lceil \frac{T}{H}\right\rceil-1\right)\Omega\left(d\sqrt{H}\right)=\Omega\left(dTH^{-\frac{1}{2}}\right).
	\end{align} 
	Intuitively, if $H,$ the number of length of each block, is smaller, the worst case regret lower bound becomes larger. But too small a block length can result in a violation of the variation budget. So we work on the total variation of $\theta_t$'s to see how small can $H$ be. 
	The total variation of the $\theta_t$'s can be seen as the total variation across consecutive blocks as $\theta_t$ remains unchanged within a single block. Observe that for any pair of $\theta,\theta'\in\left\{\pm\sqrt{d/4H}\right\}^d,$ the $\ell_2$ difference between $\theta$ and $\theta'$ is upper bounded as
	\begin{align}
		\sqrt{\sum_{i=1}^d\frac{4d}{4H}}=\frac{d}{\sqrt{H}}
	\end{align}
	and there are at most $\lfloor T/H\rfloor$ changes across the whole time horizon, the total variation is at most 
	\begin{align}
		B=\frac{T}{H}\cdot\frac{d}{\sqrt{H}}={dTH^{-\frac{3}{2}}}.
	\end{align}
	By definition, we require that $B\leq B_T,$ and this indicates that
	\begin{align}
		H\geq(dT)^{\frac{2}{3}}B_T^{-\frac{2}{3}}.
	\end{align} 
	Taking $H=\left\lceil{(dT)^{\frac{2}{3}}B_T^{-\frac{2}{3}}}\right\rceil,$ the worst case regret is 
	\begin{align}
		\Omega\left(dT\left((dT)^{\frac{2}{3}}B_T^{-\frac{2}{3}}\right)^{-\frac{1}{2}}\right)=\Omega\left(d^{\frac{2}{3}}B_T^{\frac{1}{3}}T^{\frac{2}{3}}\right).
	\end{align}
	\section{Proof of Lemma \ref{lemma:sw}}
	\label{sec:lemma:sw}
	In the proof, we denote $B(1)$ as the unit Euclidean ball, and $\lambda_\text{max}(M)$ as the maximum eigenvalue of a square matrix $M$. By folklore, we know that $\lambda_\text{max}(M) = \max_{z\in B(1)}z^\top M z$. In addition, recall the definition that $V^{-1}_{t-1} = \lambda I + \sum^{t-1}_{s = 1\vee (t-w)} X_s X_s^\top$ We prove the Lemma as follows:
	\begin{align}
		\nonumber\left\|V_{t-1}^{-1}\sum_{s=1\vee(t-w)}^{t-1}X_sX_s^{\top}\left(\theta_s-\theta_t\right)\right\| = &\left\|V_{t-1}^{-1}\sum_{s=1\vee(t-w)}^{t-1}X_sX_s^{\top}\left[\sum^{t-1}_{p=s}\left(\theta_p-\theta_{p+1}\right)\right]\right\|\\
		\label{eq:sw} = & \left\|V_{t-1}^{-1} \sum^{t-1}_{p=1\vee(t-w)}  \sum_{s=1\vee(t-w)}^{p}X_sX_s^{\top}\left(\theta_p-\theta_{p+1}\right)\right\|\\
		\label{eq:sw1} \leq &  \sum^{t-1}_{p=1\vee(t-w)} \left\| V_{t-1}^{-1} \left(\sum_{s=1\vee(t-w)}^{p}X_sX_s^{\top}\right)\left(\theta_p-\theta_{p+1}\right)\right\|\\
		\label{eq:sw2} \leq &  \sum^{t-1}_{p=1\vee(t-w)} \lambda_\text{max}\left( V_{t-1}^{-1} \left(\sum_{s=1\vee(t-w)}^{p}X_sX_s^{\top}\right)\right)\left\|\theta_p-\theta_{p+1}\right\|\\
		\label{eq:sw3} \leq &  \sum^{t-1}_{p=1\vee(t-w)} \left\|\theta_p-\theta_{p+1}\right\|.
	\end{align}	
	Equality (\ref{eq:sw}) is by the observation that both sides of the equation is summing over the terms $X_s X^\top_s (\theta_p - \theta_{p+1})$ with indexes $(s, p)$ ranging over $\{(s, p) : 1\vee (t-w) \leq s\leq p \leq t-1\}$. Inequality (\ref{eq:sw1}) is by the triangle inequality. 
	
 Inequality (\ref{eq:sw2}) is by the fact that, for any matrix $M\in \mathbb{R}^{d\times d}$ with $\lambda_\text{max}(M)\geq 0$ and any vector $y\in \mathbb{R}^d$, we have $\left\|M y\right\|_2 \leq \sqrt{\lambda_\text{max}(M^2)}\left\|y\right\|_2$.  Applying the above claim with $M = V_{t-1}^{-1} \left(\sum_{s=1\vee(t-w)}^{p}X_sX_s^{\top}\right)$ and $y = \theta_p - \theta_{p+1}$ demonstrates inequality (\ref{eq:sw2}). 

Finally, for inequality (\ref{eq:sw3}), we denote the corresponding basis for each $X_s$ as $\psi_{i(s)},$ \ie, $X_s=z_s\psi_{i(s)}=z_s\Psi e_{i(s)},$ where $e_{i}$ is the $i^{\text{th}}$ standard orthonormal basis. Let $A_1=\sum_{s=1\vee(t-w)}^{t-1}e_{i(s)}e_{i(s)}^{\top}+\lambda I$ and $A_2=\sum_{s=1\vee(t-w)}^{p}e_{i(s)}e_{i(s)}^{\top},$ it is evident that $V_{t-1}=\Psi A_1\Psi^{\top}$ and $\sum_{s=1\vee(t-w)}^{p}X_sX_s^{\top}=\Psi A_2\Psi^{\top}.$ Therefore, we have 
\begin{align}
	\nonumber\lambda_\text{max}&\left( \left(\sum_{s=1\vee(t-w)}^{p}X_sX_s^{\top}\right)V_{t-1}^{-2} \left(\sum_{s=1\vee(t-w)}^{p}X_sX_s^{\top}\right)\right)=\lambda_\text{max}\left( \Psi A_2 \Psi^{\top} (\Psi A_1 \Psi^{\top})^{-2}\Psi A_2\Psi^{\top} \right)\\
	&=\lambda_\text{max}\left( \Psi A_2A^{-2}_1 A_2\Psi^{\top} \right)=\lambda_\text{max}\left( A_2A^{-2}_1 A_2 \right)\leq 1,
\end{align} 
where we have used the fact that both $A_1$ and $A_2$ are diagonal matrix in the last step. Altogether, the Lemma is proved.
	\section{Proof of Theorem \ref{theorem:sw_deviation}}
	\label{sec:theorem:sw_deviation}
	Fixed any $\delta\in[0,1],$ we have that for any $t\in[T]$ and any $x\in D_t,$
	\begin{align}
		\nonumber\left|x^\top ( \hat{\theta}_t - \theta_t)\right|=&\left| x^\top \left( V_{t-1}^{-1}\sum_{s=1\vee(t-w)}^{t-1}X_sX_s^{\top}\left(\theta_s-\theta_t\right)\right)+ x^\top V_{t-1}^{-1}\left(\sum_{s=1\vee(t-w)}^{t-1}\eta_sX_s-\lambda\theta_t\right) \right|  \nonumber\\
		\label{eq:sw13}\leq&\left| x^\top \left( V_{t-1}^{-1}\sum_{s=1\vee(t-w)}^{t-1}X_sX_s^{\top}\left(\theta_s-\theta_t\right)\right)\right|+\left| x^\top V_{t-1}^{-1}\left(\sum_{s=1\vee(t-w)}^{t-1}\eta_sX_s-\lambda\theta_t\right) \right| \\
		\leq& \left\|x\right\|\cdot \left\|V_{t-1}^{-1}\sum_{s=1\vee(t-w)}^{t-1}X_sX_s^{\top}\left(\theta_s-\theta_t\right)\right\| + \left\|x\right\|_{V^{-1}_{t-1}}\left\|\sum_{s=1\vee(t-w)}^{t-1}\eta_sX_s-\lambda\theta_t\right\|_{V_{t-1}^{-1}}\label{eq:sw_by_CS}\\
		\leq&L \sum^{t-1}_{s = 1\vee (t-w)}\left\|\theta_s-\theta_{s+1}\right\| +  \left\|x\right\|_{V^{-1}_{t-1}} \left[ R\sqrt{d\ln\left(\frac{1+wL^2/\lambda}{\delta}\right)}+\sqrt{\lambda}S \right].\label{eq:sw-ineq},
	\end{align}
	where inequality (\ref{eq:sw13}) uses triangular inequality, inequality (\ref{eq:sw_by_CS}) follows from Cauchy-Schwarz inequality, and inequality (\ref{eq:sw-ineq}) are consequences of Lemmas \ref{lemma:sw}, \ref{lemma:sw1}.
	\section{Proof of Theorem \ref{theorem:sw_main}}
	\label{sec:theorem:sw_main}
	In the proof, we choose $\lambda$ so that $\beta \geq 1$, for example by choosing $\lambda\geq 1/S^2$. 	
	By virtue of UCB, the regret in any round $t\in[T]$ is
	\begin{align}
		\label{eq:sw15} \langle x^*_t - X_t, \theta_t\rangle & \leq  L \sum^{t-1}_{s = 1\vee (t-w)}\left\|\theta_s-\theta_{s+1}\right\| +  \langle X_t, \hat{\theta}_t\rangle + \beta\left\|X_t\right\|_{V^{-1}_{t-1}} - \langle X_t, \theta_t\rangle\\
		\label{eq:sw12} & \leq 2 L \sum^{t-1}_{s = 1\vee (t-w)}\left\|\theta_s-\theta_{s+1}\right\| + 2\beta \left\|X_t\right\|_{V^{-1}_{t-1}}.
	\end{align}	
	Inequality (\ref{eq:sw15}) is by an application of our \swofu~established in equation (\ref{eq:sw-ucb}). Inequality (\ref{eq:sw12}) is by an application of inequality (\ref{eq:sw-ineq}), which bounds the difference $| \langle X_t, \hat{\theta}_t - \theta_t\rangle |$ from above. By the evident fact that $\langle X_t, \hat{\theta}_t - \theta_t\rangle \leq 2$, we have
	\begin{equation}\label{eq:sw17}
	\langle x^*_t - X_t, \theta_t\rangle\leq 2 L \sum^{t-1}_{s = 1\vee (t-w)}\left\|\theta_s-\theta_{s+1}\right\| + 2\beta \left(\left\|X_t\right\|_{V^{-1}_{t-1}}\wedge 1\right).
	\end{equation}
	Summing equation (\ref{eq:sw17}) over $1\leq t\leq T$, the regret of the \swofu~is upper bounded as
	\begin{align}
		\nonumber\Ex\left[\nonumber\text{Regret}_T\left(\swofu\right)\right]=&\sum_{t\in[T]}\langle x^*_t- X_t,\theta_t\rangle\\
		\nonumber\leq& 2 L \left[ \sum^T_{t = 1}\sum^{t-1}_{s = 1\vee (t-w)}\left\|\theta_s-\theta_{s+1}\right\|\right] + 2\beta \sum^T_{t=1}\left(\left\|X_t\right\|_{V^{-1}_{t-1}}\wedge 1\right)\\
		\nonumber = & 2 L \left[ \sum^T_{s = 1}\sum^{(s+w)\wedge T}_{t = s+1}\left\|\theta_s-\theta_{s+1}\right\|\right] + 2\beta \sum^T_{t=1}\left(\left\|X_t\right\|_{V^{-1}_{t-1}}\wedge 1\right)\\
		\label{eq:sw11} \leq & 2 L w B_t + 2\beta \sum^T_{t=1}\left(\left\|X_t\right\|_{V^{-1}_{t-1}}\wedge 1\right).
	\end{align}
	
	What's left is to upper bound the quantity $2\beta\sum_{t\in[T]}\left(1\wedge\left\|X_t\right\|_{V^{-1}_{t-1}}\right)$. Following the trick introduced by the authors of \cite{AYPS11}, we apply Cauchy-Schwarz inequality to the term $\sum_{t\in[T]}\left(1\wedge\left\|X_t\right\|_{V^{-1}_{t-1}}\right).$
	\begin{align}
		\sum_{t\in[T]}\left(1\wedge\left\|X_t\right\|_{V^{-1}_{t-1}}\right)\leq\sqrt{T}\sqrt{\sum_{t\in[T]}1\wedge\left\|X_t\right\|^2_{V^{-1}_{t-1}}}.
	\end{align}
	By dividing the whole time horizon into consecutive pieces of length $w,$ we have 
	\begin{align}
		\label{eq:sw9}
		\sqrt{\sum_{t\in[T]}1\wedge\left\|X_t\right\|^2_{V^{-1}_{t-1}}}\leq\sqrt{\sum_{i=0}^{\lceil T/w\rceil-1}\sum_{t=i\cdot w+1}^{(i+1) w}1\wedge\left\|X_t\right\|^2_{V^{-1}_{t-1}}}.
	\end{align}
	While a similar quantity has been analyzed by Lemma 11 of \cite{AYPS11}, we note that due to the fact that $V_{t}$'s are accumulated according to the sliding window principle, the key eq. (6) in Lemma 11's proof breaks, and thus the analysis of \cite{AYPS11} cannot be applied here. To this end, we state a technical lemma based on a novel use of the Sherman-Morrison formula.
	\begin{lemma}
		\label{lemma:sw2}
		For any $i\leq \lceil T/w\rceil-1,$ 
		\begin{align*}
			\sum_{t=i\cdot w+1}^{(i+1) w}1\wedge\left\|X_t\right\|^2_{V^{-1}_{t-1}}\leq\sum_{t=i\cdot w+1}^{(i+1) w}1\wedge\left\|X_t\right\|^2_{\overline{V}_{t-1}^{-1}},
		\end{align*}
		where 
		\begin{align}
			\overline{V}_{t-1}=\sum_{s=i\cdot w+1}^{t-1}X_sX_s^{\top}+\lambda I.
		\end{align}
	\end{lemma}
	\begin{proof}{Proof of Lemma \ref{lemma:sw2}.}
		For a fixed $i\leq \lceil T/w\rceil-1,$ 
		\begin{align}
			\nonumber\sum_{t=i\cdot w+1}^{(i+1) w}1\wedge\left\|X_t\right\|^2_{V^{-1}_{t-1}}=&\sum_{t=i\cdot w+1}^{(i+1) w}1\wedge X_t^{\top}V_{t-1}^{-1}X_t\\
			\label{eq:sw7}=&\sum_{t=i\cdot w+1}^{(i+1) w}1\wedge X_t^{\top}\left(\sum_{s=1\vee(t-w)}^{t-1}X_sX_s^{\top}+\lambda I\right)^{-1}X_t.
		\end{align}
		Note that $i\cdot w+1\geq1$ and $i\cdot w+1\geq t-w~\forall t\leq(i+1)w,$ we have
		\begin{align}
			i\cdot w+1\geq 1\vee (t-w).
		\end{align}
		Consider any $d$-by-$d$ positive definite matrix $A$ and $d$-dimensional vector $y,$ then by the Sherman-Morrison formula, the matrix
		\begin{align}
			B=A^{-1}-\left(A+yy^{\top}\right)^{-1}=A^{-1}-A^{-1}+\frac{A^{-1}yy^{\top}A^{-1}}{1+y^{\top}A^{-1}y}=\frac{A^{-1}yy^{\top}A^{-1}}{1+y^{\top}A^{-1}y}
		\end{align}
		is positive semi-definite. Therefore, for a given $t,$ we can iteratively apply this fact to obtain
		\begin{align}
			\nonumber&X_t^{\top}\left(\sum_{s=i\cdot w+1}^{t-1}X_sX_s^{\top}+\lambda I\right)^{-1}X_t\\
			\nonumber=&X_t^{\top}\left(\sum_{s=i\cdot w}^{t-1}X_sX_s^{\top}+\lambda I\right)^{-1}X_t+X_t^{\top}\left(\left(\sum_{s=i\cdot w+1}^{t-1}X_sX_s^{\top}+\lambda I\right)^{-1}-\left(\sum_{s=i\cdot w}^{t-1}X_sX_s^{\top}+\lambda I\right)^{-1}\right)X_t\\
			\nonumber=&X_t^{\top}\left(\sum_{s=i\cdot w}^{t-1}X_sX_s^{\top}+\lambda I\right)^{-1}X_t+X_t^{\top}\left(\left(\sum_{s=i\cdot w+1}^{t-1}X_sX_s^{\top}+\lambda I\right)^{-1}-\left(X_{i\cdot w}X_{i\cdot w}^{\top}+\sum_{s=i\cdot w+1}^{t-1}X_sX_s^{\top}+\lambda I\right)^{-1}\right)X_t\\
			\nonumber\geq&X_t^{\top}\left(\sum_{s=i\cdot w}^{t-1}X_sX_s^{\top}+\lambda I\right)^{-1}X_t\\
			\nonumber&\vdots\\
			\label{eq:sw8}\geq&X_t^{\top}\left(\sum_{s=1\vee(t-w)}^{t-1}X_sX_s^{\top}+\lambda I\right)^{-1}X_t.
		\end{align}
		Plugging inequality (\ref{eq:sw8}) to (\ref{eq:sw7}), we have
		\begin{align}
			\nonumber\sum_{t=i\cdot w+1}^{(i+1) w}1\wedge\left\|X_t\right\|^2_{V^{-1}_{t-1}}\leq&\sum_{t=i\cdot w+1}^{(i+1) w}1\wedge X_t^{\top}\left(\sum_{s=i\cdot w+1}^{t-1}X_sX_s^{\top}+\lambda I\right)^{-1}X_t\\
			\leq&\sum_{t=i\cdot w+1}^{(i+1) w}1\wedge\left\|X_t\right\|^2_{\overline{V}_{t-1}^{-1}},
		\end{align}
		which concludes the proof.
	\end{proof}
	From Lemma \ref{lemma:sw2} and eq. (\ref{eq:sw9}), we know that
	\begin{align}
		\nonumber2\beta\sum_{t\in[T]}\left(1\wedge\left\|X_t\right\|_{V^{-1}_{t-1}}\right)\leq&2\beta\sqrt{T}\cdot\sqrt{\sum_{i=0}^{\lceil T/w\rceil-1}\sum_{t=i\cdot w+1}^{(i+1) w}1\wedge\left\|X_t\right\|^2_{\overline{V}^{-1}_{t-1}}}\\
		\label{eq:sw10}\leq&2\beta\sqrt{T}\cdot\sqrt{\sum_{i=0}^{\lceil T/w\rceil-1}2d\ln\left(\frac{d\lambda+wL^2}{d\lambda}\right)}\\
		\nonumber\leq&2\beta T\sqrt{\frac{2d}{w}\ln\left(\frac{d\lambda+wL^2}{d\lambda}\right)}.
	\end{align}
	Here, eq. (\ref{eq:sw10}) follows from Lemma 11 of \cite{AYPS11}.
	
	Now putting these two parts to eq. (\ref{eq:sw11}), we have
	\begin{align}
		\Ex\left[\nonumber\text{Regret}_T\left(\swofu\right)\right]\leq&2LwB_T+2\beta T\sqrt{\frac{2d}{w}\ln\left(\frac{d\lambda+wL^2}{d\lambda}\right)}+2T\delta\\
		=&2LwB_T+\frac{2T}{\sqrt{w}}\left(R\sqrt{d\ln\left(\frac{1+wL^2/\lambda}{\delta}\right)}+\sqrt{\lambda}S\right)\sqrt{2d\ln\left(\frac{d\lambda+wL^2}{d\lambda}\right)}+2T\delta.
	\end{align}
	Now if $B_T$ is known, we can take $w=O\left((dT)^{2/3}B_t^{-2/3}\right)$ and $\delta=1/T,$ we have
	\begin{align}
		\Ex\left[\nonumber\text{Regret}_T\left(\swofu\right)\right]=\widetilde{O}\left(d^{\frac{2}{3}}B_T^{\frac{1}{3}}T^{\frac{2}{3}}\right);
	\end{align}
	while if $B_T$ is not unknown taking $w=O\left((dT)^{2/3}\right)$ and $\delta=1/T,$ we have
	\begin{align}
		\Ex\left[\nonumber\text{Regret}_T\left(\swofu\right)\right]=\widetilde{O}\left(d^{\frac{2}{3}}\left(B_T+1\right)T^{\frac{2}{3}}\right).
	\end{align}
	\section{Proof of Lemma \ref{lemma:bob}}
	\label{sec:lemma:bob}
	For any block $i,$ the absolute sum of rewards can be written as
	\begin{align*}
		\left|\sum_{t=(i-1)H+1}^{i\cdot H\wedge T}\langle X_t,\theta_t\rangle+\eta_t\right|\leq\sum_{t=(i-1)H+1}^{i\cdot H\wedge T}\left|\langle X_t,\theta_t\rangle\right|+\left|\sum_{t=(i-1)H+1}^{i\cdot H\wedge T}\eta_t\right|\leq H+\left|\sum_{t=(i-1)H+1}^{i\cdot H\wedge T}\eta_t\right|,
	\end{align*}
	where we have iteratively applied the triangular inequality as well as the fact that $\left|\langle X_t,\theta_t\rangle\right|\leq 1$ for all $t.$
	
	Now by property of the $R$-sub-Gaussian \cite{RH18}, we have the absolute value of the noise term $\eta_t$ exceeds $2R\sqrt{\ln T}$ for a fixed $t$ with probability at most $1/T^2$ \ie, 
	\begin{align}
		\Pr\left(\left|\sum_{t=(i-1)H+1}^{i\cdot H\wedge T}\eta_t\right|\geq 2R\sqrt{H\ln\frac{T}{\sqrt{H}}}\right)\leq\frac{2H}{T^2}.
	\end{align} 
	Applying a simple union bound, we have
	\begin{align}
		\Pr\left(\exists i\in\left\lceil\frac{T}{H}\right\rceil:\left|\sum_{t=(i-1)H+1}^{i\cdot H\wedge T}\eta_t\right|\geq 2R\sqrt{H\ln\frac{T}{\sqrt{H}}}\right)\leq\sum_{i=1}^{\lceil T/H\rceil}\Pr\left(\left|\sum_{t=(i-1)H+1}^{i\cdot H\wedge T}\eta_t\right|\geq 2R\sqrt{H\ln\frac{T}{\sqrt{H}}}\right)\leq\frac{2}{T}.
	\end{align}
	Therefore, we have
	\begin{align}
		\Pr\left(Q\geq H+2R\sqrt{H\ln\frac{T}{\sqrt{H}}}\right)\leq\Pr\left(\exists i\in\left\lceil\frac{T}{H}\right\rceil:\left|\sum_{t=(i-1)H+1}^{i\cdot H\wedge T}\eta_t\right|\geq 2R\sqrt{H\ln\frac{T}{\sqrt{H}}}\right)\leq\frac{2}{T}.
	\end{align}
	The statement then follows.
\end{document}